\documentclass[letterpaper, 10 pt, conference]{article}  

\usepackage[]{amsmath}
\usepackage{bm}
\usepackage{graphicx}
\usepackage{theorem}
\usepackage[noadjust]{cite}

\newcommand{\rank}[1]{\operatorname{rank}\Big(#1\Big)}

 \newcommand{\zbf}{\mathbf{z}}

\newcommand{\Sigmabf}{\boldsymbol{\Sigma}}
\newcommand{\epsilonbf}{\boldsymbol{\epsilon}}
\newcommand{\deltabf}{\boldsymbol{\delta}}

\newcommand{\Abf}{\mathbf{A}}  \newcommand{\Cbf}{\mathbf{C}}
 \newcommand{\Ebf}{\mathbf{E}} 
  \newcommand{\Ibf}{\mathbf{I}}
  \newcommand{\Lbf}{\mathbf{L}}
\newcommand{\Mbf}{\mathbf{M}}  
 \newcommand{\Qbf}{\mathbf{Q}} \newcommand{\Rbf}{\mathbf{R}}
\newcommand{\Sbf}{\mathbf{S}}  \newcommand{\Ubf}{\mathbf{U}}
\newcommand{\Vbf}{\mathbf{V}}  
 \newcommand{\Zbf}{\mathbf{Z}}

\newcommand{\zeros}{\textbf{0}}

{\theoremstyle{break}\theoremheaderfont{\normalfont\bfseries}}
{\theoremstyle{plain}\theorembodyfont{\normalfont\rmfamily}}
{\theoremstyle{break}\theoremheaderfont{\normalfont\bfseries}
{\theoremstyle{break}\newtheorem{proposition}{Proposition}\theoremheaderfont{\normalfont\bfseries}
{\theoremstyle{break}\theoremheaderfont{\normalfont\bfseries}
{\theoremstyle{break}\newtheorem{corollary}{Corollary}\theoremheaderfont{\normalfont\bfseries}
{\theoremstyle{break}\theoremheaderfont{\normalfont\bfseries}
 
\title{ \LARGE \bf A Novel  Approach for Phase Identification in Smart Grids\\ Using Graph Theory and Principal Component Analysis }

\author{Satya Jayadev P, Aravind Rajeswaran, \\ Nirav P Bhatt, Ramkrishna Pasumarthy %
\thanks{Satya Jayadev P and Ramkrishna Pasumarthy are with Department of Electrical Engineering, and Nirav P Bhatt and Aravind Rajeswaran are with Department of Chemical Engineering, Indian Institute of Technology Madras, India.
        {\tt\small {ee15s059@ee.iitm.ac.in}, ramkrishna@ee.iitm.ac.in, niravbhatt@iitm.ac.in aravindr@smail.iitm.ac.in}}
}

\begin{document}

\maketitle
\thispagestyle{empty}
\pagestyle{empty}

\begin{abstract}
Consumers with low demand, like households, are generally supplied single-phase power by connecting their service mains to one of the phases of a distribution transformer. The distribution companies face the problem of keeping a record of consumer connectivity to a phase due to uninformed changes that happen. The exact phase connectivity information is important for the efficient operation and control of distribution system. We propose a new data driven approach to the problem based on Principal Component Analysis (PCA) and its Graph Theoretic interpretations, using energy measurements in equally timed short intervals, generated from smart meters. We propose an algorithm for inferring phase connectivity from noisy measurements. The algorithm is demonstrated using simulated data for phase connectivities in distribution networks. 

\end{abstract}

\section{Introduction}
\noindent Electrical power is generally transmitted in three phases due to technical and economic advantages over single-phase transmission \cite{Bakshi09}. At the distribution level, 3-phase power is distributed to the consumers either in one phase or three phases depending on the customer demand. For single-phase consumers, the service mains is connected to one of the phases of the distribution transformer, A phase, B phase or C phase.

Accurate phase-load connectivity information is required for balancing loads on all the phases. Balanced loads mitigate the problem of overloading on a phase. Also technical losses can be reduced leading to better efficiency in distribution.  Unbalanced loads on phases result in voltage imbalance affecting many consumers, especially with rotating electric machines. From the control point of view, the controllability limits of the power system are affected by voltage imbalance. Balanced loads on phases ensures voltage balance in the three phases.

Accurate phase connectivity information also helps in detection and localization of non-technical losses and state estimation in power distribution system.  Hence, reliable phase connectivity information is essential for efficient monitoring and optimization of distribution  networks \cite{Giannakis13}.
 
The problem faced by distribution companies is in maintaining an accurate record of the phase-load connectivity. This information might not be accurately available at all times because of changes that take place due to repairs and maintenance. Also the consumers might have the facility to switch between phases and they do so when phase tripping occurs. The distribution utility is uninformed of such changes. There are techniques such as signal injection and manual verification to determine phase connectivity but utilities refrain from them due to high costs and possible inaccuracies \cite{Caird10}. 

With the advent of smart grid technologies, distribution companies are installing smart meters at important nodal points in the network including feeders, transformers and consumer service mains. These meters can communicate readings to a central data centre with greater frequency and sometimes in real time. 

In this paper, we will  deal with identifying the phase connectivity of loads using the smart meter data. We will  propose an algorithm by combining graph theory and principal component analysis based on the principle of energy conservation. We also take into account the noise in the data arising due to technical losses, errors in smart meter readings and errors due to imperfect time synchronization of smart meter clocks.
 

\section{Related Work}
\noindent The problem of phase identification is gaining recognition in recent times with the large scale penetration of smart grids. Smart grids have lead to the search for new methods of inferring the connectivity. Chen et al. \cite{Chen11} proposed a phase identification device for phase measurement of underground transformers. Zhiyu designed a signal injector device that can be used for phase identification \cite{Zhiyu13}. However the additional hardware and staff required for these devices to work, makes them costly alternatives. Dilek et al. \cite{Dilek02} proposed a search algorithm to determine phase information using power flow analysis and load data but it ignores the noise and uncertainty in data. Arya et al. gives an approach to infer phase connectivity from time series of power measurements using mixed integer programming (MIP) \cite{Arya11} . The MIP solver is time intensive in arriving at the solution. Pezeshki and Wolfs presented a technique to identify the phases based on cross-correlation method using the time series of voltage measurements \cite{Pezeshki12}. A. Tom proposed a linear regression based algorithm which correlates between consumer voltage and substation voltage \cite{Tom13}. It requires the Geographical Information System (GIS) model which may not always be available. A method using data obtained from micro synchrophasors (uPMU) apart from the voltage magnitudes data is developed \cite{Wen15}. It proposed a brute-force search algorithm based on linear programming optimization structure, to determine phase connectivity with certain constraints on voltage magnitudes and phase angles.

We follow a method similar to that proposed in \cite{Aravind15}, in which a water distribution network is reconstructed from flow measurements.

\section{ Preliminaries}
\subsection{Graph Theory}
 The incidence matrix ($\Abf$) of a graph describes the incidence of edges on nodes and is defined as follows for a directed graph:\\
\[A_{ij} = 
\begin{cases}
\text{+1 if edge j enters node i} \\ 
\text{-1 if edge j leaves node i} \\
\text{0 if edge j is not incident on i}
\end{cases}  \]
\begin{proposition}[ Refer Theorem 8 of Chapter 3 in \cite{Andrasfai91}]
A directed graph (or a directed forest) can be uniquely constructed from an incidence matrix, provided there are no self loops. 
\end{proposition}
We make use of this proposition to determine the connectivity from incidence matrix. 

\subsection{Principal Component Analysis (PCA)}
\noindent PCA is one of the most widely used techniques in multivariate statistical analysis. PCA provides the best approximation of linear model between a set of variables, of which some are dependent on the other. A set of samples corresponding to the variables, measured at different time instances can be used to apply PCA. The linear model can be estimated even in the presence of Gaussian noise \cite{Jolliffe02}.

Let $\Zbf$ be the $(n \times N)$-dimensional matrix obtained by stacking $n$ variables of $N$ samples each. Let $\zbf_j$ be the vector of values of $n$ variables in the $j^{th}$ sample.  These variables are linearly related and described by the following model: 
\begin{equation}
\Cbf\,\zbf_j = \zeros
\end{equation}
where $\Cbf$ can be referred to as the Constraint matrix of $(n_d \times n)$ dimension where $n_d$ is the number of dependent variables. We call it the constraint matrix because it gives the physical constraints of the system.

Note that the data lies in the subspace orthogonal to the space spanned by the rows of the constraint matrix. So to represent the constraint matrix, we need a set of basis vectors orthogonal to the subspace in which the data lies. This is obtained from the eigenvectors of the Covariance matrix corresponding to the least $m$ eigenvalues, where $m$ is the number of dependent variables \cite{Aravind15}.
The Covariance matrix $\Sbf_z$ is  
\begin{equation}
\Sbf_z = \Zbf\Zbf^T
\end{equation}
The eigenvectors of the covariance matrix can be determined using Singular Value Decomposition (SVD) of the data matrix. SVD of the data matrix can be written as:
\begin{equation}
svd(\Zbf) = \Ubf_1\Sbf_1\Vbf_1^T + \Ubf_2\Sbf_2\Vbf_2^T \label{Eq3}
\end{equation}
where $\Ubf_1$ are the set of orthonormal eigenvectors corresponding to the $(n-m)$ largest eigenvalues of $\Sbf_z$ while $\Ubf_2$ are the orthogonal eigenvectors corresponding to the smallest $m$ eigenvalues of $\Sbf_z$. It has been shown that $S_R(\Ubf_2^T) \sim S_R(\Cbf)$ where $S_R(.)$ indicates the subspace spanned by the rows of $(.)$ matrix \cite{Narasimhan15}. Then $\Ubf_2^T$ satisfies the following relationship:
\begin{equation}
\Ubf_2^T\zbf=0
\end{equation}
where 
$\zbf = \left[\begin{array}{c} z_1, \, z_2, \,\cdots,  z_n \end{array}\right]^T$ is the vector of $n$ variables.
It is to be observed that the constraint matrix suffers from rotational ambiguity. 
\begin{equation}
\Qbf\Ubf_2^T\zbf=\zeros
\end{equation}
where $\Qbf$ is a non-singular matrix. So the estimated constraint matrix may not represent the physical relationships even though the correct subspace has been extracted. The estimated constraint matrix, $\Cbf$, can at best be a basis for the row space of true constraint matrix. It is shown that by partitioning $\Cbf$ into columns corresponding to dependent variables and those of independent variables, a matrix $\Rbf$ (Regression Matrix), which is unique to the system, can be obtained \cite{Aravind15}. $\Rbf$ can be computed as:
\begin{eqnarray}
 \Rbf &=& -\Cbf_d^{-1}\Cbf_i \label{Eq1}\\
\zbf_d &=& \Rbf\,\zbf_i 
 \end{eqnarray}
where $\Cbf_d$ are the columns of $\Cbf$ corresponding to dependent variables, $\Cbf_i$ are the columns of $\Cbf$ corresponding to independent variables, $\zbf_i$ is the vector of independent variables and $\zbf_d$ is the vector of dependent variables.

Many variants to PCA have been developed which pertain to different cases of error covariance matrix \cite{Narasimhan15}. Hence, this approach can be applied to noisy data.

\subsection{Integrating Graph Theory \& PCA}
\noindent In our problem, we deal with a forest of directed trees with only one parent node each and many child nodes. The challenge is to determine which child nodes are connected to which parent node. It can be observed that the sub-matrix of the incidence matrix of the forest, with rows corresponding to the parent nodes, is sufficient to infer the connectivity. This connectivity is unique as described in Proposition 1. The incidence matrix can be written as: 
\begin{equation}
\Abf = \left[ \begin{array}{c} \Abf_d \\ \Abf_i \end{array} \right] 
\end{equation}
where $\Abf_d$ are the rows corresponding to parent nodes and $\Abf_i$ are the rows corresponding to child nodes. It is pertinent to note that $\Abf_d$ only comprises of -1 and 0 as its elements. Further, it is to be noted that each column of $\Abf_d$ contains only one -1 and rest zeros.

Mathematically, the parent nodes can be considered dependent variables and child nodes to be independent variables. Hence it can be easily verified that the regression matrix $\Rbf$ which regresses the dependent variables on the independent variables is in fact the matrix $-\Abf_d$ and the uniqueness of $\Rbf$ makes it comparable element-wise to $-\Abf_d$.

Now we formulate the problem and present an algorithm which determines the regression matrix from the data pertaining to our network. 

\section{Problem Formulation}
\noindent The phase connectivity network can be represented as a forest with three trees as its components, each tree corresponding to a phase. The child nodes of a tree represent the single-phase meters connected to the corresponding phase. \footnote{This notation can be extended to 3-phase loads by considering them as three separate loads and taking the per phase readings.} The 3-phase transformer meter is considered equivalent to three single-phase meters by taking the readings corresponding to each phase, separately. Therefore, the root of each tree corresponds to a phase of the transformer meter as shown in Fig.~\ref{Figure1}.
\begin{figure}
\centering
\includegraphics[width=80mm,height=40mm]{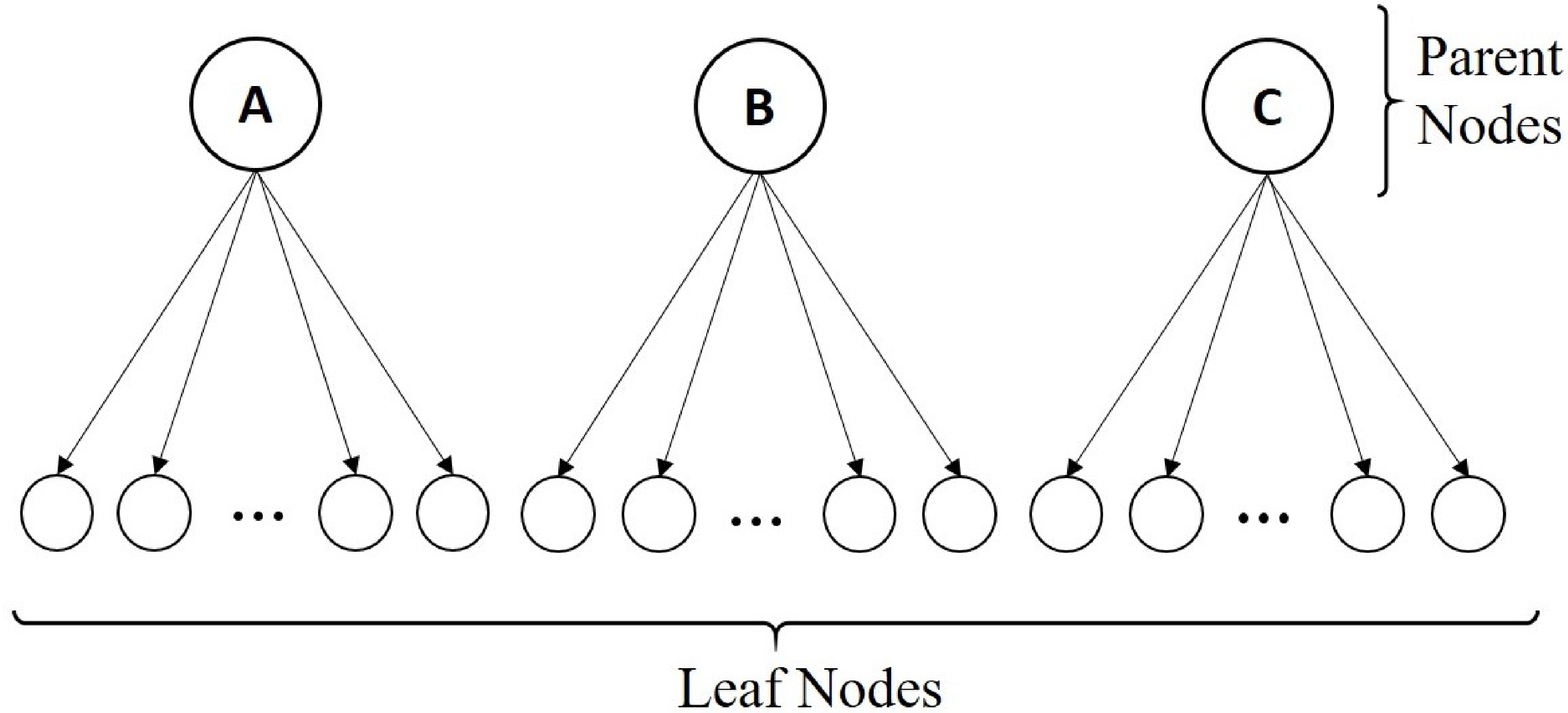}
\caption{Graph Representation of Phase Connectivity}\label{Figure1}
\end{figure}

In our approach, energy measurements in watt-hour (Wh), in equal time intervals, generally 15 minutes or 30 minutes, are collected from all the consumer meters and transformer meter to form the data matrix. Let $\Zbf$ be the data matrix, $n$ be the number of meters (or nodes), $n_d$ be the number of parent nodes and $n_i$ be the number of child nodes,
\begin{equation}
\Zbf = \left[ \begin{array}{c}z_{ij} \end{array} \right]_{(n \times N)}
\end{equation}
where $z_{ij}$ is the energy measurement corresponding to the $i^{th}$ meter in $j^{th}$ time interval. By definition,
\begin{eqnarray}
n_d &=& 3\\
n_i &=& n-3
\end{eqnarray}

The principle of conservation of energy implies that the energy supplied by each phase of transformer is equal to sum of energies consumed by all the consumers connected to that phase of the transformer, in a given time interval. In graph theory view point, this principle implies that the parent node (meter) reading is equal to the sum of child nodes (meters) readings. So there exits a linear relationship between the nodal measurements which PCA exploits to determine the constraint matrix and consequently the regression matrix. 
\begin{equation}
z_j^k = \sum\limits_{i=1}^{n_k} z_{ij}, \; \forall \, k = A,B,C \label{Eq 4}
\end{equation}
where $z_j^k$ is the energy measured in phase $k$ in the $j^{th}$ time interval, $n_k$ is the number of single-phase meters connected to phase $k$ and $z_{ij}$ is the energy measured in $i^{th}$ meter in the  $j^{th}$ time interval.

In practice, due to losses and other errors, this relationship is only approximate. We will try to infer the exact connectivity using the distributions of the noise components. The inclusion of these components in the formulation is discussed later.

Apart from the basic assumptions for PCA to work \cite{Jolliffe02}, we make the following assumptions:
\begin{enumerate}
\item There is no theft of electricity and un-metered loads are generally estimated and so are separated from the readings before applying the algorithm. 
\item The clocks of all the meters in the network are time synchronized.
\item The loads do not change phases while the $N$ measurements are recorded.
\end{enumerate}

\section{An Approach for Phase Identification}
\noindent In this section, two cases for phase identification using data are considered: (A) Noiseless data set and (B) Noisy data set.

\subsection{Noiseless Case}
\noindent With no noise, Eq.~\eqref{Eq 4} holds exactly, and we get the exact incidence sub-matrix, except for some numerical residues. 
The connectivity is determined by the following steps:
\begin{enumerate}
\item PCA is applied and the data matrix is decomposed using Eq.~\eqref{Eq3}
\item  As we know that there are only three dependent variables, we take the eigenvectors corresponding to least three eigenvalues of $\Sbf_z$ to get the constraint matrix $\Cbf_{(3 \times n)}$.
\begin{equation}
\Cbf=\Ubf_2^T
\end{equation}
\item  The columns corresponding to dependent and independent variables are separated and the regression matrix $\Rbf_{(3 \times n_i)}$ is calculated using Eq.~\eqref{Eq1}.
\item  The regression matrix is rounded off to truncate any possible numerical residues and the resultant is the incidence sub-matrix from which the phase connectivity can be inferred.
\end{enumerate}

Now the question is how many readings are required to infer the connectivity uniquely. We answer this question in Proposition 2 by establishing the lower bound of $N$ for any potential data set. 

\begin{proposition}
Let $N$ be the number of readings (or samples). In the noiseless case, the minimum number of linearly independent readings required to infer the graph uniquely is equal to $n_i$, the number of independent nodes (or consumers meters) in the network.
\begin{equation}
N\geq n_i \label{Eq2}
\end{equation}
\end{proposition}
\begin{proof} With no noise, all the samples clearly lie in the subspace represented by the constraint matrix and PCA gives a linear relationship between the variables. The linear independence of readings ensures that they span the target subspace. In such a case, the regression matrix can be written as:
\begin{equation}
\Rbf\,\Zbf^i = \Zbf^d
\end{equation}
where $\Zbf^d$ are rows of $\Zbf$ corresponding to the dependent variables and $\Zbf^i$ are the rows of $\Zbf$ corresponding to the independent variables. To obtain $\Rbf$ uniquely, $\Zbf^i$ should be a full row rank matrix. By definition of rank,
\begin{equation}
\rank{\Zbf^i} \leq min(n_i,N)
\end{equation}
For $\Zbf^i$ to be full rank, Eq.~\eqref{Eq2} should be satisfied. Hence Proposition 2 is proved. 
\end{proof}

\begin{corollary}
In the noiseless case, $n_i$ number of linearly independent energy measurements in different time intervals are sufficient to determine the phase connectivity of the loads, uniquely.  
\end{corollary}
\begin{proof}
This follows from propositions 1 and 2.
\end{proof}

\subsection{Noisy Case}
\noindent In practice, measurements are noisy and hence it is imperative to account for the various sources of noise in a distribution network. In the problem of phase identification, we account for the technical losses, random errors in smart meter readings, and smart meter clock synchronization errors.

\subsubsection*{i. Technical Losses}
Major component of technical losses in a distribution network is the copper loss and it is proportional to the square of current in the line. As the voltage is almost constant, the current varies with the load. So in a given time interval, higher the load, higher is the loss and hence the phase meter reading is always greater than the sum of consumer meter readings. This leads to
\begin{equation}
z_{i(m)}^k = z_{i(t)}^k + \sum\limits_{j=1}^{n_k} l_{ij},\,\, \forall \text{ $i$ = 1 to $N$}
\end{equation}
where $z_{i(m)}^k$ is the energy measured with loss included, in phase $k$ in $i^{th}$ time interval, $z_{i(t)}^k$ is true value of energy consumed in phase $k$ and $l_{ij}$ is the energy loss in line connecting transformer phase $k$ and consumer $j$, in $i^{th}$ time interval. Due to this loss, the noise in data is not normally distributed with zero mean. So we need to pre-process the data before applying PCA. 

\subsubsection*{ii. Random Errors in meter readings}
ANSI C12.20-2010, the latest standard for electricity meters, stipulates that electricity meters must be of 0.2 or 0.5 accuracy class \cite{ANSI2010}. That means the meter reading can be in the range of $\pm\, 0.2\%$ of true value for 0.2 accuracy class meter and in the range of $\pm\,0.5\%$ of true value for 0.5 accuracy class meter. Let us consider that all the smart meters in our network are of 0.5 accuracy class. 
\newline This error can be approximately modelled to be Gaussian with each variable having a different error variance. 
\begin{equation}
\zbf_{j(m)} = \zbf_{j(t)}+ \epsilonbf_j
\end{equation}
where $\zbf_{j(m)}$ is the vector of measured values of $n$ variables in $j^{th}$ time interval, $\zbf_{j(t)}$ is the vector of true values of $n$ variables in $j^{th}$ time interval and $\epsilonbf_j$ is the error in the reading in $j^{th}$ time interval
\begin{equation}
\epsilonbf_j \sim \mathcal{N}(\zeros,\Sigmabf_\epsilon) 
\end{equation}
where $\Sigmabf_\epsilon$ is the error covariance matrix. As the errors are not correlated, $\Sigmabf_\epsilon$ is a diagonal matrix.

\subsubsection*{iii. Clock Synchronization errors}
Clock synchronization error can also be modelled as Gaussian. The clocks of all the meters may not be synchronized perfectly leading to time intervals of measurements to be varying. The variation is generally in the order of milliseconds. \begin{equation}
t_{k+1} = t_k + \mathcal{N}(\Delta t,\sigma_t^2)
\end{equation}
where $t_k$ is the $k^{th}$ time instance, $\Delta t$ is the duration of the time interval and $\sigma_t^2$ is the variance in time interval. For example, in a 15 minutes time interval, a variance of one second will lead to an error of $0.1\%$ $(=1/(15\times60))$ and even with a drastic change (say 5 times) of load during that second, will lead to an error of only $0.5\%$.
\newline This error can be formulated similar to the previous one and superimposing them, we get
\begin{eqnarray}
\zbf_{j(m)} &=& \zbf_{j(t)} + \epsilonbf_j + \deltabf_j \\
\epsilon_j + \delta_j &\sim& \mathcal{N}(\zeros,\Sigmabf_\epsilon +\Sigmabf_\delta ) \\
\Sigmabf_e &=& \Sigmabf_\epsilon + \Sigmabf_\delta
\end{eqnarray}

In the noisy case, we propose some additional steps in the algorithm to handle the noise.
\begin{enumerate}
\item The total technical losses can be approximately calculated as:
\begin{equation}
loss_i = \sum\limits_{j=1}^{n_d} z_{ij}^d-\sum\limits_{k=1}^{n_i} z_{ik}^i, \; \forall \text{ $i$=1 to $N$}
\end{equation}
where $loss_i$ is the approximate technical loss component in $i^{th}$ time interval, $z_{ij}^d$ is the value of the $j^{th}$ dependent variable in $i^{th}$ time interval and $z_{ij}^i$ is the value of the $j^{th}$ independent variable in $i^{th}$ time interval.
\newline It is a good approximation because the magnitude of the Gaussian errors is much smaller compared to these losses. As the other technical losses are very small compared to copper losses and as the copper losses are in proportion to load, the technical losses can be separated from the parent nodes in proportion to the readings. 
\begin{equation}
\hat{z_{ij}}^d= z_{ij}^d - \frac{loss_i\,z_{ij}^d}{\sum\limits_{j=1}^{n_d} z_{ij}^d}, \; \forall \text{ $i=1$ to $N$ and $j=1$ to $n_d$ }
\end{equation}
where $\hat{z_{ij}}^d$ are the estimated values of phase meter readings, free from losses. We can apply PCA by substituting the measured values with these estimated values because the expectation of the error in these estimated values is approximately zero.
\item Now we deal with Gaussian error by using MLPCA \cite{Wentzell97a}. It is shown that if the errors in variables are not correlated and error variances are known, the constraint matrix corresponding to error free readings can be obtained by scaling the data matrix by standard deviations of corresponding errors \cite{Wentzell97a},\cite{Narasimhan15}. We take the error variances to be equal to $1\%$ of the mean readings for each variable, which is the maximum possible variance of the Gaussian errors described above.
\newline Now, MLPCA is applied as follows:
\begin{equation}
\Sigmabf_e = \text{diag}(\bar{z}_1,\,\bar{z}_2,\,\cdots,\bar{z}_n)
\end{equation}
where $\bar{z}_i$ is the mean of $N$ samples of $i^{th}$ variable
\begin{equation}
\bar{z}_i= \sum\limits_{j=1}^{N} z_{ij}
\end{equation}
Cholesky decomposition of  $\Sigmabf_e$ is given by
\begin{equation}
\Sigmabf_e = \Lbf\Lbf^T
\end{equation}
where $\Lbf$ is called Cholesky factor and it is a diagonal matrix in this case.
\newline The noisy data matrix is transformed as follows:
\begin{equation}
\Zbf_s = \Lbf^{-1}\Zbf = \Lbf^{-1}\Zbf_t + \Lbf^{-1}\Ebf
\end{equation}
where $\Ebf$ is the error matrix and $\Zbf_t$ is the data matrix with true values.
\newline The covariance matrix of the transformed data matrix is
\begin{equation}
\Sbf_{zs} = \Zbf_s\Zbf_s^T
\end{equation}
By taking expectation of the $\Sbf_{zs}$ \cite{Narasimhan15}, we get 
\begin{equation}
\Ebf(\Sbf_{zs}) =  \Mbf_z + \alpha^2\Ibf
\end{equation}
where $\Mbf_z = \Lbf^{-1}\Zbf_t\Zbf_t^T\Lbf^{-T}$, $\Ibf$ is an $n$-dimensional Identity Matrix and $\alpha^2$ is a scalar $(\alpha^2 \leq 1)$.
It is shown that by applying PCA on $\Zbf_s$, the constraint matrix that we get represents the sub-space in which the data lies \cite{Narasimhan15} . We get the constraint matrix as follows:
\begin{eqnarray}
svd(\Zbf_s) &=& \Ubf_{1s}\Sbf_{1s}\Vbf_{1s}^T + \Ubf_{2s}\Sbf_{2s}\Vbf_{2s}^T \\
\Cbf &=& \Ubf_{2s}^T \Lbf^{-1}
\end{eqnarray}
\item The columns corresponding to dependent and independent variables are separated and the regression matrix is calculated using Eq.~\eqref{Eq1}.
\item Now, the elements in $\Rbf$ are rounded off to truncate any deviations due to noise and numerical residues, by taking the elements closest to 1 to be 1 and rest 0, in each column. The resultant matrix is same as the sub-matrix which infers the connectivity between dependent and independent nodes.
\end{enumerate}
		
\section{Simulation Results}
\noindent The proposed approach is demonstrated using simulated data. Since the noiseless case is trivial, we generated noisy data for different networks and the results are presented. The network was built randomly using the random number generators in MATLAB, as follows:
\begin{enumerate}
\item Three numbers between 5 and 100 were chosen random (uniformly) to assign the number of consumers connected to each phase. 
\item The $N$ readings for each of the consumer meters were sampled from one of the three uniform distributions, with different ranges, to account for consumers with different ranges of loads. Different values of $N$ are chosen as multiples of $n_i$ .
\item Now the $N$ readings for each of the three phase meters were determined by summation of the respective consumer meter readings, connected to them. The technical losses in proportion to the consumer readings were added to the summation. 
\newline World Bank data indicates that countries with good power infrastructure have average losses in the range of $2\%$ to $10\%$ \cite{WorldBank}. Also the losses are projected to come down further in a smart grid set up. In our Simulation, we considered two cases with losses in the range of $2\%$ to $5\%$ and with losses in the range of $5\%$ to $10\%$ of the energy transmitted. 
\item To account for the other errors, we added Gaussian noise to all the readings with mean equal to the reading and standard deviation in the range $0.5\%$ to $1\%$ of the reading.
\end{enumerate}

The algorithm is then applied on 100 such generated data sets, with different number of readings and loss components and the results were  noted. These simulations were carried out in MATLAB version R2014a. The time taken for the algorithm to give the solution was also noted in all the cases (Windows 10, Intel i5-4200U 1.64 Ghz processor, 6 GB RAM). Now we plot our results as follows:
\begin{enumerate}
\item Fig.~\ref{Figure2} shows the time taken to arrive at the solution against the number of nodes for different number of readings, with losses in the range $2\%$ to $5\%$.
\begin{figure}[h!]
\centering
\includegraphics[width=80mm,height=50mm]{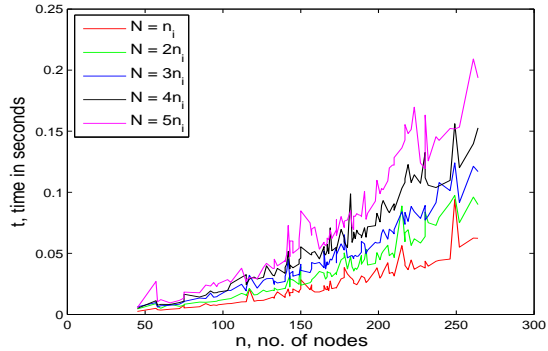}
\caption{No. of nodes (n) Vs Simulation Time (t in seconds) for losses in the range $2\%$ to $5\%$}\label{Figure2}
\end{figure}
\item Fig.~\ref{Figure3} shows the time taken to arrive at the solution against the number of nodes for different number of readings, with losses in the range $5\%$ to $10\%$.
\begin{figure}
\centering
\includegraphics[width=80mm,height=50mm]{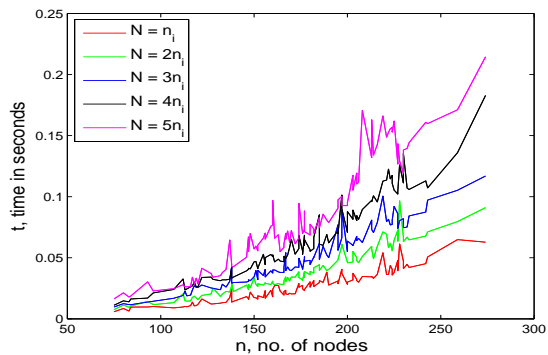}
\caption{No. of nodes (n) Vs Simulation Time (t in seconds) for losses in the range $5\%$ to $10\%$}\label{Figure3}
\end{figure}
\item Fig.~\ref{Figure4} shows the success rate in \% for 100 test cases against the number of readings expressed as ratio $N/n_i$.
\begin{figure}
\centering
\includegraphics[width=80mm,height=50mm]{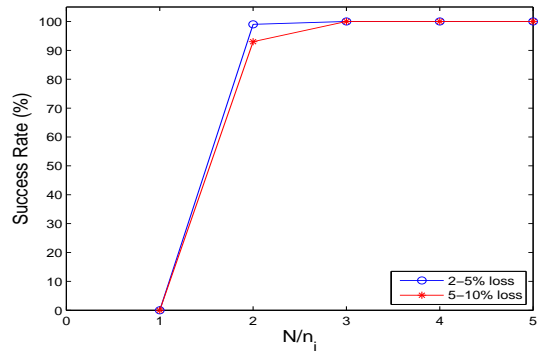}
\caption{ $N/n_i$ Vs Success Rate (in \%)} \label{Figure4}
\end{figure}
\end{enumerate}

Plots Fig:~\ref{Figure2} and Fig:~\ref{Figure3} indicate that the time taken increases with the number of nodes and readings but it is not dependent on the errors. This is apparent because the cost of computations in the algorithm depends only on the number of nodes and readings and is independent of the error. The important point to be observed is that the time taken is in the order of milliseconds. It shows that our algorithm is time efficient compared to search algorithms like the one proposed in \cite{Arya11}, which takes time in the order of seconds to determine connectivity.

Plot Fig:~\ref{Figure4} indicates that the success rate is 100$\%$ in all the cases when $N \geq 3n_i$. Therefore $3n_i$ readings, satisfying the assumptions, are sufficient to infer the phase connectivity exactly.
\section*{VII. Conclusions and Future Work}
\noindent In this paper, we show that the phase connectivity problem can be solved through a novel data-driven approach. Our formulation enables the use of PCA and its graph theoretic interpretation to infer the connectivity directly. 

The Simulations results show that as long as the errors are within the limits assumed, which is so in most of the cases, the connectivity can be exactly determined. Also the time taken is observed to be in the order of milliseconds establishing the efficiency of the algorithm.

In future, we will extend this technique to reconstructing the complete distribution network, from sub-station to consumers, given the meter readings at all nodal points. We would also like to handle the problem of localizing non-technical losses, especially theft, and also address the problem of corrupted or missing measurements.
\section*{Acknowledgment}
\noindent We would like to thank Prof. Shankar Narasimhan of IIT Madras for his valuable inputs. The finance support to Satya Jayadev P. and Aravind Rajeswaran from Data Science Initiative Grant of IIT Madras, and Nirav Bhatt from Department of Science \& Technology, India through INSPIRE Faculty Fellowship is acknowledged.

\end{document}